\documentclass{article}

\usepackage[final, nonatbib]{neurips_2019}

\usepackage[utf8]{inputenc} % allow utf-8 input
\usepackage[T1]{fontenc}    % use 8-bit T1 fonts
\usepackage{hyperref}       % hyperlinks
\usepackage{url}            % simple URL typesetting
\usepackage{booktabs}       % professional-quality tables
\usepackage{amsfonts}       % blackboard math symbols
\usepackage{nicefrac}       % compact symbols for 1/2, etc.
\usepackage{microtype}      % microtypography

\usepackage{amsmath}  % added by Melika 
\usepackage{xcolor}

\usepackage{tikz}
\usetikzlibrary{arrows}

\newcommand{\Real}{\mathbb{R}}

\newcommand{\diag}{{\rm diag}}

\newcommand{\tnorm}[1]{\|#1\|_2}

\newtheorem{theorem}{Theorem}[section]

\newenvironment{proof}[1][Proof]{\begin{trivlist}
\item[\hskip \labelsep {\bfseries #1}]}{\end{trivlist}}

\newcommand{\bfzero}{\mathbf{0}}

\newcommand{\qed}{\nobreak \ifvmode \relax \else
      \ifdim\lastskip<1.5em \hskip-\lastskip
      \hskip1.5em plus0em minus0.5em \fi \nobreak
      \vrule height0.75em width0.5em depth0.25em\fi}

\newcommand{\herm}{^{\rm H}}

\def\beq{\begin{equation}}
\def\eeq{\end{equation}}
\def\beqstar{\begin{equation*}}
\def\eeqstar{\end{equation*}}

\newcommand{\Th}{\Theta}

\newcommand{\ones}{\mathbf{1}}

\newcommand{\mcX}{\mathcal{X}}

\newcommand{\hbf}{\mathbf{h}}
\newcommand{\zb}{\mathbf{z}}
\newcommand{\tim}{^{(k)}}

\def\bm{\begin{matrix}}
\def\bbm{\begin{bmatrix}}
\def\ebm{\end{bmatrix}}

\newcommand{\norm}[1]{\|#1\|}
\newcommand{\Ab}{\mathbf{A}}
\newcommand{\Ib}{\mathbf{I}}
\newcommand{\Tb}{\mathbf{T}}
\newcommand{\Wb}{\mathbf{W}}
\newcommand{\Fb}{\mathbf{F}}
\newcommand{\bb}{\mathbf{b}}
\newcommand{\hb}{\mathbf{h}}
\newcommand{\xb}{\mathbf{x}}
\newcommand{\yb}{\mathbf{y}}
\newcommand{\Cb}{\mathbf{C}}

\newcommand{\Db}{\mathbf{D}}
\newcommand{\Thb}{\mathbf{\Theta}}
\newcommand{\thb}{\mathbf{\theta}}
\DeclareMathOperator*{\I}{\mathbf{I}}

\newcommand{\dhid}{n} % number of hidden states
\newcommand{\mbbR}{\mathbb{R}}
\newcommand{\din}{m}
\newcommand{\dout}{p}
\newcommand{\fb}{\mathbf{f}}
\newcommand{\cb}{\mathbf{c}}
\newcommand{\tran}{^{\text{\sf T}}}

\title{Input-Output Equivalence of Unitary and Contractive RNNs}

\author{%David S.~Hippocampus\thanks{Use footnote for providing further information about author (webpage, alternative address)---\emph{not} for acknowledging funding agencies.} \\ Department of Computer Science\\ Cranberry-Lemon University\\ Pittsburgh, PA 15213 \\\texttt{hippo@cs.cranberry-lemon.edu} \\
  Melikasadat Emami\\
  Dept. ECE\\
  UCLA\\
  \texttt{emami@ucla.edu}\\
  \And 
  Mojtaba Sahraee-Ardakan\\
  Dept. ECE\\
  UCLA\\
  \texttt{msahraee@ucla.edu}\\
  \And 
  Sundeep Rangan\\
  Dept. ECE\\
  NYU\\
  \texttt{srangan@nyu.edu}\\
  \And 
  Alyson K. Fletcher\\
  Dept. Statistics\\
  UCLA\\
}

\begin{document}
\maketitle
\begin{abstract}
Unitary recurrent neural networks (URNNs) 
have been proposed as a method to overcome
the vanishing and exploding gradient problem 
in modeling data with long-term dependencies. 
A basic question is
how restrictive is the
unitary constraint on the possible input-output
mappings of such a network?
This work 
shows that for any contractive RNN 
with ReLU activations, there is a
URNN with at most twice the number of hidden
states and the identical input-output
mapping.  Hence, with ReLU activations,
URNNs are as expressive
as general RNNs.  In contrast, 
for certain smooth activations, it is shown 
that the input-output mapping of an RNN
cannot be matched with a URNN, even with an arbitrary
number of states.  The theoretical results are supported by experiments on modeling of
slowly-varying dynamical systems.
\end{abstract}

\section{Introduction}
Recurrent neural networks (RNNs)
-- originally proposed in the late 1980s
\cite{rumelhart1988learning,elman1990finding}
-- refer to a widely-used and powerful class of models for
time series and sequential data.
In recent years,
RNNs have become particularly important in
%in numerous areas, including computer vision \cite{krizhevsky2012imagenet}, \cite{donahue2015long}, 
speech recognition 
\cite{graves2013speech,hinton2012deep}
and natural language processing 
\cite{collobert2011natural,bahdanau2014neural,sutskever2014sequence} tasks.

A well-known challenge in training 
recurrent neural networks is the \emph{vanishing}
and \emph{exploding} gradient problem
\cite{bengio1993problem,pascanu2013difficulty}.
RNNs have a transition matrix 
that maps the hidden state at one time to the next time.
When the transition matrix has an induced norm
greater than one, the RNN may become unstable.
In this case, small 
perturbations of the input at some time 
can result in a
change in the output 
that grows exponentially over the subsequent time.
This instability leads to a
so-called exploding gradient.
Conversely, when the norm is less than one,
perturbations can decay exponentially so inputs
at one time have negligible effect in the 
distant future.
As a result, the loss surface associated with 
RNNs can have steep walls that may be 
difficult to minimize.
Such problems are particularly acute in systems
with long-term dependencies, where the output sequence
can depend strongly on the input sequence many time steps in the past.

Unitary RNNs (URNNs) \cite{arjovsky2016unitary} 
is a simple and commonly-used 
approach to mitigate the vanishing and
exploding gradient problem. The basic idea
is to restrict the
transition matrix to be unitary (an orthogonal 
matrix for the real-valued case).
The unitary transitional matrix is then combined with
a non-expansive activation such as a ReLU or sigmoid.
As a result, the overall transition mapping cannot
amplify the hidden states, thereby eliminating the
exploding gradient problem.
In addition, since all the singular
values of a unitary matrix equal 1, the
transition matrix does not attenuate the hidden
state, potentially mitigating the vanishing gradient
problem as well.  (Due to activation, the hidden state
may still be attenuated).
Some early work in URNNs suggested that they could 
be more effective than other methods, 
such as long short-term memory (LSTM) architectures
and standard RNNs, for certain 
learning tasks involving long-term dependencies  \cite{jing2017tunable, arjovsky2016unitary} -- see a short
summary below.

Although URNNs may improve the stability of the network
for the purpose of optimization, a basic issue
with URNNs is that the unitary
contraint may potentially reduce the set of 
input-output mappings that the network can model.
This paper seeks to rigorously
characterize \emph{how restrictive
the unitary constraint is on an RNN}\@.
We evaluate this restriction by comparing the set
of input-output mappings achievable with URNNs
with the set of mappings from all RNNs.
As described below, we restrict our attention
to RNNs that are contractive in order to avoid
unstable systems.

We show three key results:
\begin{enumerate}
\item Given any contractive RNN with $n$ hidden states
and ReLU activations, there exists a URNN with at most
$2n$ hidden states and the identical input-ouput mapping.

\item This result is tight in the sense that, given
any $n > 0$, there exists at least one contractive RNN
such that any URNN with the same input-output mapping
must have at least $2n$ states.

\item The equivalence of
URNNs and RNNs depends on the activation.  For example,
we show that there exists a contractive
RNN with \emph{sigmoid} activations such that there
is no URNN with any finite number of states that exactly
matches the input-output mapping.
\end{enumerate}

The implication of this result is that, for RNNs with
ReLU activations, there is no loss in the 
\emph{expressiveness} of model when imposing the unitary
constraint.  As we discuss below,
the penalty is a two-fold increase in the 
number of parameters.

Of course, the expressiveness
of a class of models is only one factor
in their real performance.  
Based on these results alone,
one cannot determine if URNNs will outperform
RNNs in any particular task.  
Earlier works have found examples where URNNs offer
some benefits over LSTMs and RNNs~\cite{arjovsky2016unitary,wisdom2016full}.
But in the simulations below concerning
modeling slowly-varying nonlinear dynamical
systems, we see that
URNNs with $2n$ states perform approximately
equally to RNNs with $n$ states.  

Theoretical results on generalization error are an active subject area in deep neural networks. Some measures of model complexity such as \cite{neyshabur2017exploring} are related to the spectral norm of the transition matrices. For RNNs with non-contractive matrices, these complexity bounds will grow exponentially with the number of time steps. In contrast, since unitary matrices can bound the generalization error, this work can also relate to generalizability.

\subsection*{Prior work}
The vanishing and exploding gradient problem
in RNNs has been known almost
as early as RNNs themselves
\cite{bengio1993problem,pascanu2013difficulty}.
It is part of a larger problem
of training models that can capture long-term
dependencies, and   
several proposed methods address this issue. 
Most approaches use some form of gate vectors to control the information flow inside the hidden states, the most
widely-used being LSTM networks~\cite{hochreiter1997long}.
Other gated models include Highway networks \cite{srivastava2015highway}
and gated recurrent units (GRUs)~\cite{Cho_2014}.
L1/L2 penalization on gradient norms and gradient clipping were proposed to solve the exploding gradient problem in \cite{pascanu2013difficulty}. With L1/L2 penalization, capturing long-term dependencies is still challenging since the regularization term quickly kills the information in the model.  A more recent work
\cite{pennington2018emergence} has successfully trained
very deep networks by carefully adjusting 
the initial conditions to impose an approximate
unitary structure of many layers. 
%Another class of methods try to enforce orthogonality to the weight matrices. %By limiting the singular values of the state transition matrix, unitary/orthogonal weights can bound the long-term gradient. 

Unitary evolution RNNs (URNNs) are a more
recent approach first proposed in
\cite{arjovsky2016unitary}.
Orthogonal constraints were also considered in
the context of associative memories \cite{white2004short}.
One of the technical difficulties is to 
efficiently parametrize the set of unitary matrices.
The numerical simulations in this work focus
on relatively small networks, where the parameterization
is not a significant computational issue.
Nevertheless, for larger numbers of hidden states,
several approaches have been proposed.
The model in \cite{arjovsky2016unitary} parametrizes the transition matrix as a product of reflection, diagonal, permutation, and Fourier transform matrices. This model spans a subspace of the whole unitary space, thereby limiting the expressive power of RNNs.  The work
\cite{wisdom2016full} overcomes this issue by optimizing over full-capacity unitary matrices.
A key limitation in this work, however, is that the projection of weights on to the unitary space is not computationally efficient. 
A tunable, efficient parametrization of unitary matrices is proposed in \cite{jing2017tunable}. This model provides the computational complexity of $O(1)$ per parameter. The unitary matrix is represented as a product of rotation matrices and a diagonal matrix. By grouping specific rotation matrices, the model provides tunability of the span of the unitary space and enables using different capacities for different tasks. Combining the parametrization in \cite{jing2017tunable} for unitary matrices and the ``forget'' ability of the GRU structure, \cite{Cho_2014,jing2019gated} presented an architecture that outperforms conventional models in several long-term dependency tasks. Other methods such as orthogonal RNNs proposed by \cite{mhammedi2017efficient} showed that the unitary constraint is a special case of the orthogonal constraint. By representing an orthogonal matrix as a product of Householder reflectors, we are able span the entire space of orthogonal matrices. Imposing hard orthogonality constraints on the transition matrix limits the expressiveness of the model and speed of convergence and performance may degrade \cite{vorontsov2017orthogonality}. %\Melika{Do we need to cite \cite{white2004short} somewhere?}
%\alert{[Check that correctness of the previous sentence was not compromised.  (Rearranged to prevent sentence from starting with cite.)]}

%One final theoretical point:  It is known in \cite{funahashi1993approximation} that any dynamical
%system can be approximated arbitrarily well
%by a RNN with a sufficiently high state dimension.
%\alert{[Preceding sentence is strangely placed.  Can this information be conveyed somewhere where it is contrasting with the potential lack of approximability with URNNs?]}

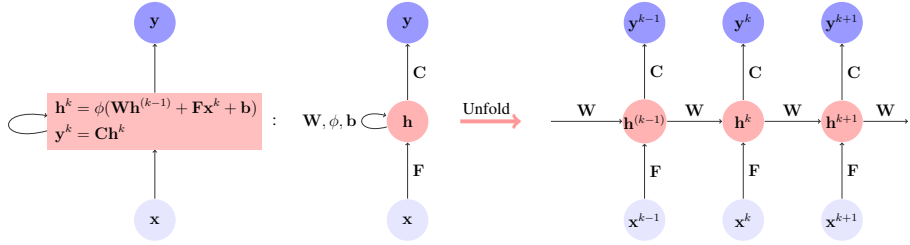
\begin{figure*}\centering
\scalebox{0.35}{
\begin{tikzpicture}[thick,->,draw=black!80, node distance=2.5cm, scale=1, every node/.style={scale=1.5, font=\large}]
\tikzstyle{every pin edge}=[<-,line width=3pt]
\tikzstyle{neuron}=[circle,fill=black!25,minimum size=30pt,inner sep=0pt]
\tikzstyle{input neuron}=[neuron, fill=blue!10];
\tikzstyle{hidden neuron}=[neuron, fill=red!30];
\tikzstyle{output neuron}=[neuron, fill=blue!40];
\tikzstyle{annot} = [text width=4em, text centered]
\tikzstyle{null} = [draw = none , fill= none]
\tikzstyle{line} = [draw, black, -latex']
\tikzstyle{block}=[rectangle,fill=red!25,minimum size=40pt]
\node [input neuron] (x) {$\xb$};
\node[block, above of=x](b) {$\begin{aligned}
     \hb^{k} &=\phi(\Wb \hb^{(k-1)} + \Fb \xb^k +\bb)\\  \yb^k &= \Cb \hb^k
\end{aligned}$};
\node [output neuron, above of=b] (y) {$\yb$};
%\node[annot, left of=b, , node distance=5cm](hs){Hidden States};
%\node[annot, left of=x, , node distance=5cm](in){Input};
%\node[annot, left of=y, , node distance=5cm](hs){Output};

\path (x) edge (b);
\path (b) edge (y);
\path[->] (b.185) edge [out=195, in=170,distance=2cm] node[] {}(b.177);

\node[annot, right of=b, node distance =3cm](mid1){:};

\node [hidden neuron, right of=mid1, node distance=3.4cm] (h) {$\hb$};
\node [input neuron, below of=h, node distance=2.5cm] (xx) {$\xb$};
\node [output neuron, above of=h] (yy) {$\yb$};

\path (xx) edge [right] node{$\Fb$} (h);
\path (h) edge [loop left] node{$\Wb, \phi, \bb$} (h);
\path (h) edge [right] node{$\Cb$} (yy);

\node[annot, right of=h, node distance=0.5cm](mid2){};
\node[null, right of=mid2, node distance=2.5cm](mid3){};
\path (mid2) edge [above, line width=1.5mm, draw=red!40] node[text width = 4em]{Unfold} (mid3);

\node [null, right of=mid3, node distance = 0.5cm](a1){};
\node [hidden neuron, right of=a1](ht-1) {$\hb^{(k-1)}$};
\node [hidden neuron, right of=ht-1] (ht) {$\hb^k$};
\node [hidden neuron, right of=ht] (ht+1) {$\hb^{k+1}$};

\node [input neuron, below of=ht-1] (xt-1) {$\xb^{k-1}$};
\node [input neuron, below of=ht] (xt) {$\xb^k$};
\node [input neuron, below of=ht+1] (xt+1) {$\xb^{k+1}$};

\node [annot, right of=ht+1](a2){};

\node [output neuron, above of=ht-1](yt-1) {$\yb^{k-1}$};
\node [output neuron, above of=ht] (yt) {$\yb^k$};
\node [output neuron, above of=ht+1] (yt+1) {$\yb^{k+1}$};

% connections 
\path (xt-1) edge [right] node{$\Fb$} (ht-1);
\path (xt) edge [right] node{$\Fb$} (ht);
\path (xt+1) edge [right] node{$\Fb$}(ht+1);

\path (ht-1) edge [right] node{$\Cb$} (yt-1);
\path (ht) edge [right] node{$\Cb$} (yt);
\path (ht+1) edge [right] node{$\Cb$} (yt+1);

\path (ht-1) edge [above] node{$\Wb$}  (ht);
\path (ht) edge [above] node{$\Wb$}  (ht+1);
\path (a1) edge [above] node{$\Wb$}  (ht-1);
\path (ht+1) edge [above] node{$\Wb$}  (a2);
\end{tikzpicture}
}
\caption{Recurrent Neural Network (RNN) model.
\label{fig:rnn} }
\end{figure*}

\section{RNNs and Input-Output Equivalence}

\paragraph*{RNNs.}
We consider recurrent neural networks (RNNs) 
representing sequence-to-sequence mappings of the form
\begin{subequations} \label{eq:RNN}
\begin{align}
    &\hb^{(k)} = \phi(\Wb \hb^{(k-1)} + \Fb \xb^{(k)} + \bb), \quad \hb^{(-1)} = \hb_{-1}, 
    \label{eq:RNN_state} \\
    &\yb^{(k)} = \Cb \hb^{(k)}, \label{eq:RNN_output}
\end{align}
\end{subequations}
parameterized by 
$\Thb = (\Wb, \Fb, \bb, \Cb,\hb_{-1})$.
The system is shown in Fig.~\ref{fig:rnn}.
The system maps a sequence of inputs
$\xb^{(k)} \in \mbbR^{\din}$, $k=0,1,\ldots,T-1$
to a sequence of outputs $\yb^{(k)}
\in \mbbR^{\dout}$.
In equation~\eqref{eq:RNN},
$\phi$ is the activation function
(e.g. sigmoid or ReLU); 
$\hb^{(k)} \in \mbbR^{\dhid}$ is an internal
or hidden state;
$\Wb \in \mbbR^{\dhid \times \dhid}, \Fb \in \mbbR^{\dhid \times \din}$, and $\Cb \in \mbbR^{\dout \times \dhid}$ are the hidden-to-hidden, input-to-hidden, and hidden-to-output weight matrices respectively; and
$\bb$ is the bias vector.  
We have considered the initial condition,
$\hb_{-1}$, as part of the parameters,
although we will often take $\hb_{-1}=0$.
Given a set of parameters $\Thb$, 
we will let
\beq \label{eq:gtheta}
    \yb = G(\xb, \Thb)
\eeq
denote the resulting sequence-to-sequence mapping.
Note that the number of time samples, $T$,
is fixed throughout our discussion.

Recall \cite{strang1993introduction}
that a matrix 
$\Wb$ is \emph{unitary} if $\Wb\herm\Wb = \Wb\Wb\herm =\Ib$.
When a unitary matrix is real-valued,
it is also called \emph{orthogonal}.
In this work, 
we will restrict our attention to real-valued matrices,
but still use the term unitary for consistency with the URNN literature.
A \emph{Unitary RNN} or URNN is simply
an RNN \eqref{eq:RNN}
with a unitary state-to-state transition
matrix $\Wb$.
A key property of unitary 
matrices is that
they are \emph{norm-preserving}, meaning that
$\|\Wb\hb^{(k)}\|_2=\|\hb^{(k)}\|_2$.
In the context of \eqref{eq:RNN_state}, the unitary
constraint implies that the transition matrix
does not amplify the state.

\paragraph*{Equivalence of RNNs.}
Our goal is to understand the extent to which
the unitary constraint in a URNN
restricts the set of input-output mappings.
To this end, we say
that the RNNs for two parameters $\Thb_1$ and
$\Thb_2$ are \emph{input-output equivalent}
if the sequence-to-sequence mappings are identical,
\beq \label{eq:equiv}
    G(\xb,\Thb_1)=G(\xb,\Thb_2) \mbox{ for all }
    \xb=(\xb^{(0)},\ldots,\xb^{(T-1)}).
\eeq
That is, for all input sequences $\xb$, the two systems
have the same output sequence.  Note that the hidden
internal states $\hb^{(k)}$ in the two systems
may be different.  We will also say that 
two RNNs are \emph{equivalent on a set of $\mcX$}
of inputs if \eqref{eq:equiv} holds
for all $\xb \in \mcX$.

It is important to recognize that 
input-output equivalence does \emph{not}
imply that the parameters $\Thb_1$ and $\Thb_2$
are identical.  For example, consider the case
of linear RNNs where the activation in
\eqref{eq:RNN} is the identity,
$\phi(\zb) = \zb$.  Then, for any invertible
$\Tb$, the transformation
\beq \label{eq:Ttrans}
    \Wb \rightarrow \Tb\Wb\Tb^{-1},\quad
    \Cb \rightarrow \Cb\Tb^{-1}, \quad
    \Fb \rightarrow \Tb\Fb, \quad
    \hb_{-1} \rightarrow \Tb\hb_{-1},
\eeq
results in the same input-output mapping.
However, the internal states $\hb^{(k)}$
will be mapped to $\Tb\hb^{(k)}$.
The fact that many parameters can lead to 
identical input-output mappings will be key to finding
equivalent RNNs and URNNs.

\paragraph*{Contractive RNNs.}  The
\emph{spectral norm} \cite{strang1993introduction}
of a matrix $\Wb$ is 
the maximum gain of the matrix $\|\Wb\| := \max_{\hb \neq 0} \frac{\|\Wb\hb\|_2}{\|\hb\|_2}$.
In an RNN \eqref{eq:RNN},
the spectral norm $\|\Wb\|$ measures how much the 
transition matrix can amplify the hidden state.
For URNNs, $\|\Wb\|=1$.  We will say an
RNN is \emph{contractive} if $\|\Wb\|< 1$,
\emph{expansive} if $\|\Wb\|>1$, and
\emph{non-expansive} if $\|\Wb\|\leq 1$.
In the sequel, we will restrict our attention
to contractive and non-expansive RNNs.  
In general, given an expansive
RNN, we cannot expect to find an equivalent URNN\@.
For example, suppose $\hbf^{(k)}=h^{(k)}$ is scalar.
Then, the transition matrix $\Wb$ is also
scalar $\Wb=w$ and $w$ is expansive if and only if
$|w|>1$.  Now suppose 
the activation is a ReLU $\phi(h)=\max\{0,h\}$.
Then, it is possible that a constant
input $x^{(k)}=x_0$ can result in an
output that grows exponentially with time:
$y^{(k)} = \mathrm{const} \times w^k$.  
Such an exponential 
increase is not possible with a URNN\@. 
%In general, given an expansive RNN, 
%it is not possible to find an equivalent
%URNN\@.  So,
We consider only non-expansive RNNs
in the remainder of the paper.
Some of our results will also need the 
assumption that the activation 
function $\phi(\cdot)$ in \eqref{eq:RNN}
is non-expansive:
\begin{align*}
    \norm{\phi(\xb)- \phi(\yb)}_2 \leq \norm{\xb-\yb}_2,
    \qquad \mbox{for all $\xb$ and $\yb$}.
\end{align*}
This property is satisfied
by the two most common activations,
sigmoids  and ReLUs.

\paragraph*{Equivalence of Linear RNNs.}
To get an intuition of equivalence, it is useful
to briefly review the concept in the case
of linear systems \cite{kailath1980linear}.  
Linear systems are RNNs \eqref{eq:RNN}
in the special case where the
activation function is identity, $\phi(\zb)=\zb$;
the initial condition is zero, $\hb_{-1}=0$; and
the bias is zero, $\bb=0$.
In this case, it is well-known that
two systems are input-output equivalent 
if and only if they 
have the same \emph{transfer function},
\beq \label{eq:Hs}
    H(s) := \Cb(s\Ib-\Wb)^{-1}\Fb.
\eeq
In the case of scalar inputs
and outputs, $H(s)$ is a
%polynomial
rational 
function of the complex variable
$s$ with numerator and denominator degree
of at most $n$, the dimension of the hidden
state $\hb^{(k)}$.  Any state-space system
\eqref{eq:RNN} that achieves a particular transfer
function is called a \emph{realization}
of the transfer function.  Hence two linear
systems are equivalent if and only if they 
are the realizations of the same transfer function.

A realization is called \emph{minimal} if it
is not equivalent some linear system
with fewer hidden states.  
A basic property of realizations of linear systems
is that they
are minimal if and only if they are \emph{controllable}
and \emph{observable}.  The formal definition
is in any linear systems text, e.g. \cite{kailath1980linear}.  Loosely, controllable
implies that all internal states can be reached
with an appropriate input and 
observable implies that all hidden states can
be observed from the ouptut.  In absence
of controllability and observability, some
hidden states can be removed while maintaining
input-output equivalence.

\section{Equivalence Results for RNNs with ReLU Activations}

Our first results consider contractive RNNs 
with ReLU activations.
For the remainder of the section,
we will restrict our attention to the 
case of zero initial conditions,
$\hb^{(-1)}=0$ in \eqref{eq:RNN}. 

\begin{theorem}\label{thm:1}
Let $ \yb =  G (\xb, \Thb_c)$ be a contractive RNN with ReLU activation and states of dimension $n$. 
Fix $M > 0$ and let $\mcX$ be the set of all sequences such that $\tnorm{\xb\tim}\leq M < \infty$
for all $k$.  Then there exists a URNN with state dimension $2n$ and parameters 
$\Thb_u = (\Wb_u, \Fb_u, \bb_u, \Cb_u)$ such that for all $\xb\in \mcX$, $G(\xb,\Th_c) =  G (\xb, \Thb_u)$. Hence the input-output mapping is matched for bounded 
inputs.
\end{theorem}
\begin{proof} See Appendix~\ref{sec:thm_relu_pf}.
\end{proof}

Theorem \ref{thm:1} shows that for any contractive RNN with ReLU activations, there exists a URNN with at most twice the number of hidden states and the
identical input-output mapping.
Thus, there is
%, in principle,
no loss in the set of input-output mappings with URNNs
relative to general contractive RNNs on bounded inputs.

The penalty for using RNNs is the two-fold
increase in state dimension,
which in turn increases the number of parameters
to be learned.  We can estimate this increase
in parameters as follows:
The raw number of parameters for an RNN \eqref{eq:RNN}
with $n$ hidden states, $p$ outputs and $m$ inputs 
is $n^2+(p+m+1)n$.  However, for ReLU activations,
the RNNs are equivalent under the transformations
\eqref{eq:Ttrans} using diagonal positive $\Tb$.
Hence, the number of degrees of freedom of 
a general RNN is at most $d_{\rm rnn} = n^2+(p+m)n$.
We can compare this value to a URNN with $2n$
hidden states.
The set of $2n \times 2n$ unitary $\Wb$
has $2n(2n-1)/2$ degrees of freedom
\cite{stewart1980efficient}.  Hence, the total
degrees of freedom in a URNN with $2n$ states
is at most $d_{\rm urnn} = n(2n-1) + 2n(p+m)$.
We conclude that a URNN with $2n$ hidden states
has slightly fewer than twice the number of parameters
as an RNN with $n$ hidden states.

We note that there are cases that the contractivity assumption is limiting, however, the limitations may not always be prohibitive. We will see in our experiments that imposing the contractivity constraint can improve learning for RNNs when models have sufficiently large numbers of time steps. Some related results where bounding the singular values help with the performance can be found in \cite{vorontsov2017orthogonality}.

%We note that the construction is not restricted to ReLU. The key idea here is that half of the states are killed by ReLU and a linear combination of the other half can exactly match the output of the original system. We further look at how tight is the assumption for the number of states in the unitary system. The following theorem shows twice the number of states for the URNN is necessary to get the identical input-output mapping.

We next show a converse result.
\begin{theorem}\label{thm:relu_conv}
For every positive $n$, 
there exists a contractive RNN with ReLU nonlinearity and state dimension $n$ such that every equivalent URNN has at least $2n$ states.
\end{theorem}
\begin{proof}
See Appendix~\ref{sec:relu_conv_pf} in the Supplementary Material.
\end{proof}
The result shows that the $2n$ achievability
bound in Theorem~\ref{thm:1} is tight, 
at least in the worst case.
In addition, the RNN 
constructed in the proof of Theorem~
\ref{thm:relu_conv} is not particularly
pathological.
We will show in our simulations
in Section~\ref{sec:sim}
that URNNs typically need twice the number
of hidden states to achieve comparable modeling
error as an RNN\@.

\section{Equivalence Results for RNNs with Sigmoid Activations}

%An interesting property of
Equivalence
between RNNs and URNNs depends on the particular
activation.  Our next result shows that
with sigmoid activations, URNNs are, in general,
never exactly equivalent to RNNs, even
with an arbitrary number of states.

We need the following technical definition:
Consider an RNN \eqref{eq:RNN} with 
a standard sigmoid activation $\phi(z)=1/(1+e^{-z})$.
If $\Wb$ is non-expansive, then a
simple application of the contraction mapping
principle shows that for any constant 
input $x\tim=x^*$,
there is a fixed point in the hidden state 
$\hb^* = \phi(\Wb\hb^* + \Fb\xb^* + \bb)$.
We will say that the RNN is controllable
and observable at $\xb^*$ if the linearization
of the RNN around $(\xb^*,\hb^*)$ is controllable
and observable.

\begin{theorem}\label{thm:sig_conv}
There exists a contractive RNN 
with sigmoid activation function $\phi$ with
the following property:  If a
URNN is controllable and observable
at any point $\xb^*$, then the URNN
cannot be equivalent to the RNN
for inputs $\xb$ in the neighborhood of
$\xb^*$.
\end{theorem}
\begin{proof}  See Appendix~\ref{sec:sig_conv_pf} in the Supplementary Material.
\end{proof}

The result provides a 
converse on equivalence:  Contractive 
RNNs with sigmoid activations are not 
in general equivalent to URNNs, even if we allow
the URNN to have an arbitrary number of hidden
states.  Of course, the approximation error
between the URNN and RNN may go to zero
as the URNN hidden dimension goes to infinity
(e.g., similar to the approximation results in
\cite{funahashi1993approximation}).
However, exact equivalence is not possible 
with sigmoid activations, unlike with ReLU activations.
Thus, there is fundamental difference
in equivalence for smooth and non-smooth
activations.

We note that the  fundamental distinction between Theorem \ref{thm:1} and the opposite result in Theorem~\ref{thm:sig_conv} is that the activation is smooth with a positive slope.  With such
activations, you can linearize the system,
and the eigenvalues of the transition matrix become visible in the input-output mapping.  In contrast,
ReLUs can zero out states and
suppress these eigenvalues.
This is a key insight
of the paper and a further contribution
in understanding nonlinear systems.

\section{Numerical Simulations} \label{sec:sim}
In this section, we numerically 
compare the modeling ability of RNNs
and URNNs where the true system is a
contractive RNN with long-term dependencies.
Specifically,
we generate data from multiple instances
of a synthetic RNN 
where the parameters in \eqref{eq:RNN}
are randomly generated.  For the true system,
we use $m=2$
input units, $p=2$ output units,
and $n=4$ hidden units at each time step.
The matrices $\Fb$, $\Cb$ and $\bb$
are generated as i.i.d.\ Gaussians.
We use a random transition matrix,
\beq \label{eq:Wrand}
    \Wb = \Ib-\epsilon\Ab\tran\Ab/\|\Ab\|^2,
\eeq
where $\Ab$ is Gaussian i.i.d.\ matrix
and $\epsilon$ is a small value, taken here
to be $\epsilon=0.01$.  The matrix
\eqref{eq:Wrand} will be contractive
with singular values in $(1-\epsilon,1)$.
By making $\epsilon$ small, the states
of the system will vary slowly, hence
creating long-term dependencies.
In analogy with linear systems, the time
constant will be approximately 
$1/\epsilon=100$ time steps.
We use ReLU activations.  To avoid
degenerate cases where the outputs are
always zero, the biases $\bb$
are adjusted to ensure that the each hidden
state is on some target $60\%$ of the time
using a similar procedure as in 
\cite{fletcher2018inference}.

The trials have $T=1000$ time steps,
which corresponds to 10 times the time constant $1/\epsilon=100$ of the system.
We added noise to the output of this system such that the signal-to-noise ratio (SNR) is 15 dB or 20 dB\@.   In each trial, we generate
700 training samples and 300 test sequences
from this system.  

Given the input and the output data of this contractive RNN, we attempt to
learn the system with: (i) standard
RNNs, (ii) URNNs, and (iii) LSTMs.
The hidden states in the model are varied
in the range 
$n = [2,4,6,8,10,12,14]$, which include
values both above and below the true
number of hidden states $n_{\rm true}=4$.
We used mean-squared error as the loss function.  
Optimization is performed using Adam \cite{kingma2014adam} optimization with
a batch size = 10 and learning rate = 0.01.
All models are implemented in the
Keras package in Tensorflow.
The experiments are done over 30 realizations of the original contractive system.

For the URNN learning,
of all the proposed algorithms for enforcing the unitary constraints on transition matrices during training \cite{jing2017tunable, wisdom2016full,arjovsky2016unitary,mhammedi2017efficient},
we chose to project the transition matrix on the full space of unitary matrices after each iteration using singular value decomposition (SVD).
Although SVD requires $\mathcal{O}(n^3)$ computation for each projection, for our choices of hidden states it performed faster than the aforementioned methods.

Since we have training noise and since
optimization algorithms can get stuck
in local minima, we cannot expect
``exact" equivalence between the learned
model and true system as in the theorems.  
So, instead,
we look at the test error as a measure
of the closeness of the learned model
to the true system.
Figure~\ref{fig:rsq_snr20} on the left shows the test $R^2$ for a Gaussian i.i.d.\ input and output with SNR = 20 dB for RNNs, URNNs, and LSTMs. The red dashed line corresponds to the optimal $R^2$ achievable at the given noise level.  

Note that even though the true RNN has $n_{\rm true}=4$ hidden states, the RNN
model does not obtain the optimal
test $R^2$ at $n=4$.  This is not due
to training noise, since the RNN is able
to capture the full dynamics  when we over-parametrize the system to $n\approx 8$
hidden states.  The test error in the
RNN at lower numbers of hidden states is likely
due to the optimization being caught
in a local minima.  

What is important for this work though is to compare the URNN test error 
with that of the RNN\@.  We observe
that URNN requires approximately twice
the number of hidden states to obtain
the same test error as achieved by an
RNN\@.  To make this clear,
the right plot shows the same performance data with number of states adjusted for URNN\@.
Since our theory indicates that a URNN with $2n$ hidden states is as powerful as an RNN with $n$ hidden states, 
we compare a URNN with $2n$ hidden units directly with an RNN with $n$ hidden units. We call this the adjusted hidden units.
We see that the URNN and RNN have similar
test error when we appropriately
scale the number of hidden units
as predicted by the theory.

For completeness, the left plot in Figure~\ref{fig:rsq_snr20} also shows
the test error with an LSTM\@.
It is important to note that the URNN has almost the same performance as an LSTM with considerably smaller number of parameters.

Figure \ref{fig:rsq_snr15} shows similar 
results for the same task with SNR = 15 dB\@. For this task, the input is \emph{sparse} Gaussian i.i.d., i.e.\ Gaussian with some probability $p=0.02$ and $0$ with probability $1-p$. The left plot shows the $R^2$ vs.\ the number of hidden units for RNNs and URNNs and the right plot shows the same results once the number of hidden units for URNN is adjusted.

We also compared the modeling ability of URNNs and RNNs using the Pixel-Permuted MNIST task. Each MNIST image is a $28 \times 28$ grayscale image with a label between 0 and 9.  A fixed random permutation is applied to the pixels and each pixel is fed to the network in each time step as the input and the output is the predicted label for each image \cite{arjovsky2016unitary,jing2017tunable, vorontsov2017orthogonality}.

We evaluated various models on the Pixel-Permuted MNIST task using validation based early stopping. Without imposing a contractivity constraint during learning, the RNN is either unstable or requires a slow learning rate. Imposing a contractivity constraint improves
the performance.  Incidentally, using a URNN improves the performance further. Thus, contractivity can improve learning
for RNNs when models have sufficiently large numbers of time steps.

\begin{figure}
    \centering
    \includegraphics[width=13cm]{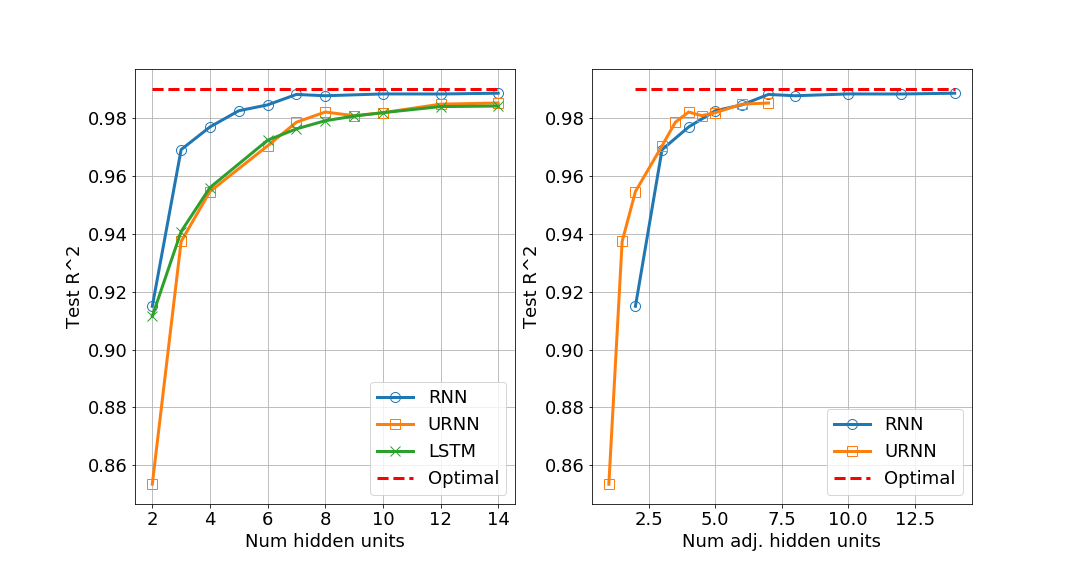}
    \caption{Test $R^2$ on synthetic data
    for a Gaussian i.i.d.\ input and output SNR=20~dB.}
    \label{fig:rsq_snr20}
\end{figure}

\begin{figure}
    \centering
    \includegraphics[width=13cm]{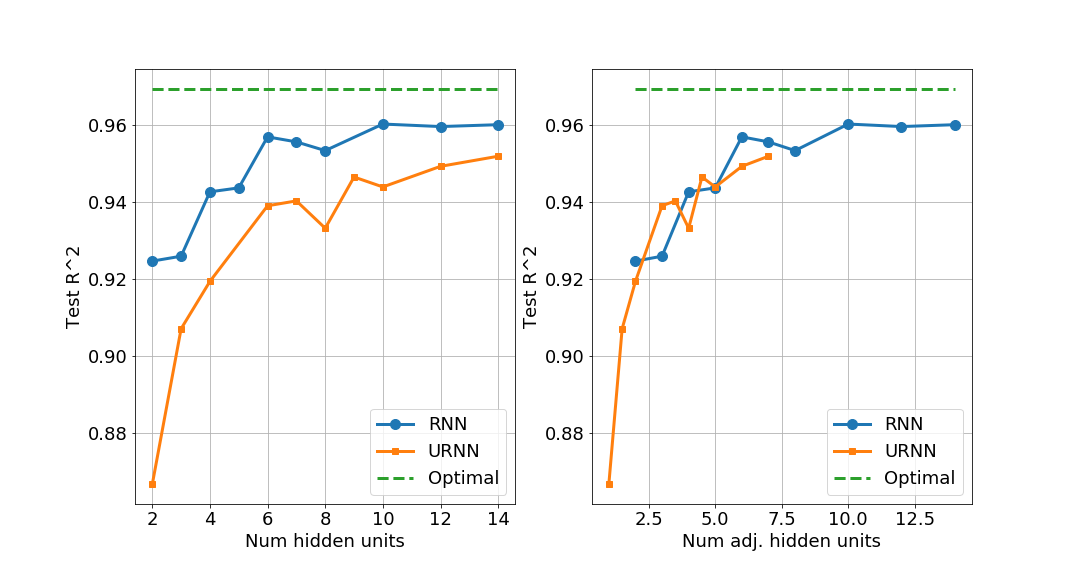}
    \caption{Test $R^2$ on synthetic data
    for a Gaussian i.i.d.\ input and output SNR=15~dB.}
    \label{fig:rsq_snr15}
\end{figure}

\begin{figure}
\centering
%{r}{0.5\textwidth}
    %\vspace{-0.5cm}
    \includegraphics[scale = 0.4]{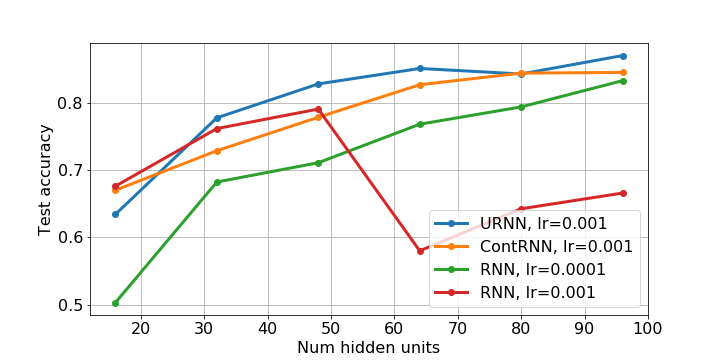}
    %\vspace{-0.25cm}
    \caption{
    Accuracy on Permuted MNIST task for various models trained with
    RMSProp, validation-based
    early termination, and
     initial learning rate {\tt lr}.
    (1) URNN model: RNN model
    with unitary constraint;
    (2) ContRNN:
    RNN with a contractivity constraint;
    (3 \& 4) RNN model with no contractivity
    or unitary constraint (two learning rates).
    We see contractivity improves
    performance, and unitary constraints
    improve performance further.
    }
    \label{fig:pmnist}
\end{figure}

\section{Conclusion}
Several works empirically show that using unitary recurrent neural networks improves the stability and performance of the RNNs. In this work, we study how restrictive it is to use URNNs instead of RNNs. We show that URNNs are at least as powerful as contractive RNNs in modeling input-output mappings if enough hidden units are used. More specifically, for any contractive RNN we explicitly construct a URNN with twice the number of states of the RNN and identical input-output mapping. We also provide converse results for the number of state and the activation function needed for exact matching. We emphasize that although it has been shown that URNNs outperform standard RNNs and LSTM in many tasks that involve long-term dependencies, our main goal in this paper is to show that from an approximation viewpoint, URNNs are as expressive as general contractive RNNs. By a two-fold increase in the number of parameters, we can use the stability benefits they bring for optimization of neural networks.

\section*{Acknowledgements}
The work of M. Emami, M. Sahraee-Ardakan, A.\ K.\ Fletcher was supported in part by the National Science Foundation under Grants 1254204 and 1738286, and the Office of Naval Research under Grant N00014-15-1-2677. S. Rangan was supported in part by the National Science Foundation under Grants 1116589, 1302336, and 1547332, NIST, the industrial affiliates of NYU WIRELESS, and the SRC.

\appendix
\section{Proof of Theorem~\ref{thm:1}}
\label{sec:thm_relu_pf}

The basic idea is to construct a URNN with $2n$ states such that first $n$ states match the states of RNN and the last $n$ states are always zero. 
To this end, consider any contractive RNN,
\begin{align*}
    \hb_c\tim = \phi(\Wb_c \hb_c^{(k-1)} + \Fb_c \xb\tim + \bb_c), \quad
    \yb\tim = \Cb_c \hb_c\tim,
\end{align*}
where $\hb\tim\in \Real^n$.
Since $\Wb$ is contractive, we have $\|\Wb\|\leq \rho$
for some $\rho < 1$.  Also, for a ReLU activation,
$\|\phi(\zb)\|\leq \|\zb\|$ for all pre-activation
inputs $\zb$.  Hence, 
\begin{align*}
    \|\hb_c\tim\|_2 &=
    \| \phi(\Wb_c \hb_c^{(k-1)} + \Fb_c \xb\tim + \bb_c)\|_2 
    \leq \| \Wb_c \hb_c^{(k-1)}  + \Fb_c \xb\tim + \bb_c\|_2 \\
    &\leq \rho\|\hb_c^{(k-1)}\|_2 + \|\Fb_c\|\|\xb\tim\|_2
    + \|\bb_c\|_2.
\end{align*}
Therefore, with bounded inputs, $\|\xb\tim\|\leq M$,
we have the state is bounded,
\beq \label{eq:hbnd}
    \|\hb\tim\|_2 \leq \frac{1}{1-\rho}
    \left[ \|\Fb_c\|M + \|\bb_c\|_2 \right] =: M_h.
\eeq
%where $\hb_c^t,\bb_c \in \mathbb{R}^\dhid$, $\Wb_c \in \mathbb{R}^{\dhid \times \dhid}$, $\Fb_c \in \mathbb{R}^{\dhid \times \din}$, and $\Cb_c \in \mathbb{R}^{\dout \times \dhid}$.\\
We construct a URNN as,
\begin{align*}
    \hb_u\tim  = \phi(\Wb_u \hb_u^{(k-1)} 
    + \Fb_u \xb\tim + \bb_u), \quad
    \yb\tim = \Cb_u \hb_u\tim
\end{align*}
where the parameters are of the form,
\beq 
    \hb_u = \bbm \hb_1 \\ \hb_2 \ebm\in \Real^{2n},
    \quad
    \Wb_u = \bbm \Wb_1 , \Wb_2 \\ \Wb_3, \Wb_4 \ebm,
    \quad
    \Fb_u = \bbm \Fb_c \\ \bfzero \ebm,
    \quad \bb_u = \bbm \bb_c \\ \bb_2 \ebm.  
\eeq
Let $\Wb_1 = \Wb_c$. Since $\norm{\Wb_c}<1$, we have $\I - \Wb_c\tran\Wb_c \succeq 0$. Therefore, there exists $\Wb_3$ such that $\Wb_3\tran\Wb_3 = \I - \Wb_c\tran\Wb_c$. With this choice of $\Wb_3$, the first $n$ columns of $\Wb_u$ are orthonormal. Let $\bbm \Wb_2 \\ \Wb_4 \ebm$ extend these to an orthonormal basis for $\Real^{2n}$. Then, the matrix $\Wb_u$ will be orthonormal.

Next, let $\bb_2 = -M_h \ones_{\dhid \times 1}$, 
where $M_h$ is defined in \eqref{eq:hbnd}.
We show by induction that for all $k$,
\beq \label{eq:RNNequivalence}
 \hb_1\tim = \hb_c\tim,\quad \hb_2\tim = \bfzero.
\eeq
If both systems are initialized at zero, \eqref{eq:RNNequivalence} is satisfied at $k=-1$. 
Now, suppose this holds up to time $k-1$. Then,
 \begin{align*}
     \hb_1\tim &= \phi(\Wb_1 \hb_1^{(k-1)} + \Wb_2 \hb_2^{(k-1)}+ \Fb_c \xb\tim + \bb_c) \\
     &=
     \phi(\Wb_1 \hb_1^{(k-1)} + 
     \Fb_c \xb\tim + \bb_c)= \hb_c\tim,
 \end{align*}
where we have used the induction hypothesis that
$\hb_2^{(k-1)}=\bfzero$.
For $\hb_2\tim$, note that 
\begin{align}
    \norm{\Wb_3 \hb_1^{(k-1)}}_\infty \leq \norm{\Wb_3 \hb_1^{(k-1)}}_2 \leq \norm{\hb_1^{(k-1)} }
    \leq M_h,
\end{align}
where the last step follows from \eqref{eq:hbnd}.
Therefore,
\begin{align}
    \Wb_3 \hb_1^{(k-1)} + \Wb_4 \hb_2^{(k-1)} + \bb_2 = \Wb_3 \hb_1^{(k-1)} - M \ones_{\dhid \times 1} \leq \bfzero.
\end{align}
Hence with ReLU activation $\hb_2\tim = \phi(\Wb_3 \hb_1^{(k-1)} + \Wb_4 \hb_2^{(k-1)} + \bb_2) = \bfzero$.
By induction, \eqref{eq:RNNequivalence} holds
for all $k$.  Then, if we define
$\Cb_u = [\Cb_c \bfzero]$, we have the output
of the URNN and RNN systems are identical
\[
    \yb\tim_u = \Cb_u\hb_u\tim  = \Cb_c\hb_1\tim =
    \yb\tim_c.
\]
This shows that the systems are equivalent.

\bibliographystyle{plain}
\bibliography{ref}

%\vspace{20.65cm}
\pagebreak

{\bf \Large Supplementary Material}

%\appendix

\section{Converse Theorem Proofs} 
\label{sec:proofs}

\subsection{Proof of Theorem~\ref{thm:relu_conv}}
\label{sec:relu_conv_pf}

First consider the case when $n=1$ with
scalar inputs and outputs.
Let $\thb_c = (w_c, f_c, b_c, c_c)$ be the parameters
of a contractive RNN with $f_c=c_c=1$, $b_c=0$
and $w_c \in (0,1)$.
Hence, the contractive RNN is given by
\begin{align} \label{eq:rnnc_scalar}
    h_c\tim = \phi(w_c h_c^{(k-1)} + x\tim),
    \quad   y\tim = h_c\tim,
\end{align}
and $\phi(z)=\max\{0,z\}$ is the ReLU activation.
Suppose $\Thb_u$ are the parameters of an 
equivalent URNN\@.  If $\Thb$ has less than $2n=2$
states, it must have $n=1$ state. 
Let the equivalent URNN be 
\begin{align}
    h_u\tim &= \phi(w_uh_u^{(k-1)} + f_u x\tim + b_u),
    \quad
    y\tim = c_u h_u\tim,
\end{align}
for some parameters $\Thb_u=(w_u,f_u,b_u,c_u)$.
Since $w_u$ is orthogonal, either $w_u=1$
or $w_u=-1$.  Also, either $f_u > 0$ or $f_u < 0$.
First, consider the case when $w_u=1$ and $f_u > 0$.
Then, there exists a large enough input $x\tim$ such that for all time steps $k$, both systems are operating in the active phase of ReLU\@. Therefore, we have two equivalent linear systems,
\begin{align}
    \text{contractive RNN:}\quad h_c\tim &= w_c h_c^{(k-1)} + x\tim, \quad  y\tim = h_c\tim\\
    \text{URNN:}\quad h_u\tim &= h_u^{(k-1)} + f_u x\tim + b_u, \quad
    y\tim = c_u h_u\tim.
\end{align}
In order to have identical input-output mapping for these linear systems for all $x$, it is required that $w_c=1$, which is a contradiction. 
The other cases $w_c=-1$ and $f_u<0$ can be
treated similarly.
Therefore, at least $n=2$ states are needed for the URNN to match the contractive RNN with $n=1$
state.
 
For the case of general $n$, consider the 
contractive RNN,
\begin{align} \label{eq:RNN_stateb}
    \hb\tim = \phi(\Wb \hb^{(k-1)} + \Fb \xb\tim + \bb),  \quad 
    \yb\tim = \Cb \hb\tim,
\end{align}
where $\Wb = \diag(w_c, w_c, ..., w_c)$, 
$\Fb = \diag(f_c, f_c, ..., f_c)$, 
$\bb = b_c \ones_{\dhid \times 1}$, and 
$\Cb = \diag(c_c, c_c, ..., c_c)$. 
This system is separable in that if $\yb=G(\xb)$
then $y_i = G(x_i,\theta_c)$ for each 
input $i$.  A URNN system will need 2 states
for each scalar system requiring a total of
$2n$ states.

\subsection{Proof of Theorem~\ref{thm:sig_conv}}
\label{sec:sig_conv_pf}

We use the same scalar contractive RNN 
\eqref{eq:rnnc_scalar}, but with
a sigmoid activation $\phi(z)=1/(1+e^{-z})$.
Let $\Thb=(\Wb_u,\fb_u,\cb_u,\bb_u)$
be the parameters of any URNN with
scalar input and outputs.  Suppose
the URNN is controllable and observable
at an input value $x^*$.
Let $h_c^*$ and $\hb_u^*$ be, respectively, the
fixed points of the hidden states
for the contractive RNN and URNN:
\begin{align}
     \text{contractive RNN:} \quad h_c^* &= \phi(w_c h_c + x^* ), \quad\\
     \text{URNN:} \quad \hb_u^* &= 
     \phi(\Wb_u \hb_u + \fb_u x^* + \bb_u).
\end{align}
We take the linearizations \cite{vidyasagar2002nonlinear}
of each system around its fixed point and apply a small perturbation $\Delta x$ around $x^*$. Therefore, we have two linear systems with identical input-output mapping given by,
\begin{align}
    \text{contractive RNN:}\quad 
    \Delta h_c\tim  &= 
    d_c(w_c \Delta h^{(k-1)}+ \Delta x\tim),
    \quad y\tim = \Delta h_c\tim+h_c^*, \\
    \text{URNN:} \quad 
    \Delta \hb_u\tim &= \Db_u (\Wb_u \Delta \hb_u^{(k-1)} + \fb_u\tran \Delta x\tim  ),
    \quad 
    y\tim = \cb_u\tran \Delta \hb_u +    \cb_u\tran\hb_u^*,
\end{align}
where 
\[
    d_c = \phi'(z_c^*=w_ch_c^* + x^*), 
    \quad
    \Db_u = 
    \phi'(\Wb_u \hb_u^* + \fb_u x^* +\bb_u ),
\]
are the derivatives of the activations
at the fixed points.
Since both systems are controllable and observable, their dimensions must be the same and the eigenvalues of the transition 
matrix must match. 
In particular, the URNN must be scalar,
so $\Wb_u=w_u$ for some scalar $w_u$.
For orthogonality, either $w_u=1$
or $w_u=-1$.  We look at the $w_u=1$ case;
the $w_u=-1$ case is similar.
Since the eigenvalues of the transition
matrix must match we have,
\beq \label{eq:eval_comp}
    d_c w_c = d_u \Rightarrow 
    \phi'(w_c h_c^* + x^*) w_c =
    \phi'(h_u^* + f_u x^* +b_u).
\eeq
where $h_u^*$ and $h_c^*$ are the solutions
to the fixed point equations:
\beq \label{eq:hfix_comp}
    h_c^* = \phi(w_c h_c + x^* ), \quad    
      h_u^* = 
     \phi(h_u^* + f_b x^* + b_u).
\eeq
Also, since two systems have the same
output,
\beq \label{eq:ycomp}
    h_c^* = c_uh_u^*.
\eeq
Now, \eqref{eq:eval_comp} must hold
at any input $x^*$ where the
URNN is controllable and observable.
If the URNN is controllable and observable
at some $x^*$, it is
is controllable and observable in a 
neighborhood of $x^*$.
Hence, \eqref{eq:eval_comp} and 
\eqref{eq:ycomp}
holds in some neighborhood of $x^*$.
To write this mathematically, define
the functions,
\beq \label{eq:gmatch}
    g_c(x^*) := \left[ \begin{array}{c}
       w_c\phi'(w_ch_c^*+x^*) \\ h_c^* 
       \end{array} \right],
      \quad
    g_u(x^*) := \left[ \begin{array}{c}
       \phi'(h_u^*+f_ux^*+b_u) \\ c_uh_u^* 
       \end{array} \right],
\eeq
where, for a given $x^*$,
$h_u^*$ and $h_c^*$ are the solutions
to the fixed point equations \eqref{eq:hfix_comp}.
We must have that $g_c(x^*)=g_u^*(x^*)$
for all $x^*$ in some neighborhood.
Taking derivatives of \eqref{eq:gmatch}
and using the fact that $\phi(z)$ being a sigmoid, one can show that this matching
can only occur when,
\[
    w_c = 1, \quad b_u = 0, \quad c_u = 1.
\]
This is a contradiction
since we have assumed that the RNN
system is contractive which
requires $|w_c|=1$.
\end{document}